\documentclass{article}

\usepackage{amssymb}
\usepackage{microtype}
\usepackage{graphicx}
\usepackage{subfigure}
\usepackage{booktabs} 
\usepackage{multirow}
\usepackage{float}
\usepackage{tabularx}
\usepackage{hyperref}

\usepackage[accepted]{icml2019}
\usepackage{amsmath}
\usepackage{amssymb}

\usepackage{amsthm}
\newtheorem{theorem}{Theorem}[section]

\newtheorem{lemma}[theorem]{Lemma}

\icmltitlerunning{NodeDrop: A Method for Finding Sufficient Network Architecture Size with Improved Generalization}
\begin{document}

\twocolumn[
\icmltitle{NodeDrop: A Condition for Reducing Network Size without Effect on Output}



\icmlsetsymbol{equal}{*}

\begin{icmlauthorlist}
\icmlauthor{L. Jensen}{equal,bu}
\icmlauthor{J. Harer}{equal,bu}
\icmlauthor{S. Chin}{bu,cbmm,cmsa,bbn}
\end{icmlauthorlist}

\icmlaffiliation{bu}{Dept. of Computer Science, Boston University}
\icmlaffiliation{cbmm}{Center for Brains, Minds and Machines (CBMM),  Dept. of Brain and Cognitive Sciences, MIT}
\icmlaffiliation{cmsa}{Center of Mathematical Sciences and Applications (CMSA), Harvard}
\icmlaffiliation{bbn}{BBN Technologies}

\icmlcorrespondingauthor{Louis Jensen}{ljjensen@bu.edu}
\icmlcorrespondingauthor{Jacob Harer}{jharer@bu.edu}

\icmlkeywords{Machine Learning, Deep Learning}

\vskip 0.3in
]



\printAffiliationsAndNotice{\icmlEqualContribution} 

\begin{abstract}

Determining an appropriate number of features for each layer in a neural network is an important and difficult task. This task is especially important in applications on systems with limited memory or processing power. Many current approaches to reduce network size either utilize iterative procedures, which can extend training time significantly, or require very careful tuning of algorithm parameters to achieve reasonable results. In this paper we propose NodeDrop, a new method for eliminating features in a network. With NodeDrop, we define a condition to identify and guarantee which nodes carry no information, and then use regularization to encourage nodes to meet this condition. We find that NodeDrop drastically reduces the number of features in a network while maintaining high performance, reducing the number of parameters by a factor of $114$x for a VGG like network on CIFAR10 without a drop in accuracy.


\end{abstract}

\section{Introduction}
\label{Intro}
A prime difficulty in neural network design is the appropriate tuning of network architectures. Choosing a size for each layer of a neural network is usually done by rough estimate, trial, and error. This imprecise process can often lead to network designs larger than needed to perform a particular task. Although the capacity for training large and complex networks grows with improving graphics processing unit (GPU) technology, designing too large a network can result in applications impracticable for general hardware use. Mobile devices and embedded systems limit compute, memory, and storage consumption, and as a result can only run small, minimally designed networks. A designer aiming to create such a minimal network is faced with the time-consuming task of manually tuning the number of neurons in each layer. This tuning process can result in many extended tuning experiments in order to balance the space and performance of the neural network.

The issues involved with using deep neural networks (DNN's) on constrained systems has inspired significant research. One interesting area of research is the design of systems which can automatically prune a network's parameters. Ideally these techniques can still maintain high performance while pruning as many parameters as possible, ensuring the network can fit on smaller systems. Many state-of-the-art methods for network pruning generally involve an iterative process of repeatedly pausing training, pruning parameters, and resuming training  in order for the network to reconverge. Such iterative procedures can lead to long training times. Other techniques use regularization in order to eliminate nodes. The final performance of these networks are often highly variable with the hyper-parameters of the algorithm. Thus, while these techniques do offer parameter reduction benefits, the network designer will still be faced with similar difficulties as before: a time consuming training process and a potential hyper-parameter tuning headache.

We address the problem of parameter reduction with our novel NodeDrop technique, which prunes the network during training. The NodeDrop technique only drops nodes which carry no information and drops them fluidly during the training process.


First, we formally define the conditions necessary to guarantee a neuron carries no information. We then propose a simple variant of $L_1$ regularization which drives nodes toward this condition. Second, we extend the NodeDrop technique to networks which use batch normalization \cite{Ioffe_2015}. 
We test our technique on modern architectures for the MNIST, CIFAR10, and CIFAR100 datasets, and show that we are able to drop a significant number of nodes without a loss in performance. Our method requires no iterative retraining and only a modest increase in training time. We demonstrate effective results with a wide range of hyperparamaters, indicating our method does not require precise hyperparameter tuning. At best case we produce a network which reduces the number of parameters by $93.27$, $99.12$, and $87.82$ percent for MNIST, CIFAR10, and CIFAR100 respectively, with no perceivable loss in performance.



\section{Related Works}
\label{Related_Works}

\subsection{Pruning}
Network pruning comprises a set of techniques which take a pretrained network and then prune off connections using some heuristic. This is usually followed by a retraining of the network and sometimes by an iterative process of pruning and retraining the network several times. Pruning techniques first appeared in the 1990s, with the first instances using second order gradients of connections to determine which neurons should be pruned \cite{Hassibi_1993, LeCun_1990, Reed_1993}. More recent approaches have taken on a wide array of methods for determining which connections should be pruned. These approaches include correlation \cite{Sun_2015, Han_2016, Srinivas_2015}, regularization \cite{Han_2015, Li_2017}, particle filtering on misclassification rate \cite{Anwar_2017}, low rank approximation \cite{Denton_2014}, vector quantization \cite{Gong_2014} and tensor decomposition \cite{Kim_2015}.

All network pruning techniques suffer from extended training time due to the iterative retraining of the network. This can lengthen training times significantly, and often makes tuning the various parameters in each method a lengthy chore. 

\subsection{Regularization}
A more recently developed approach to network parameter reduction is to disable parameters through regularization. A majority of these techniques have focused on the sparsification of network connections using a group sparsity approach \cite{Wen_2016, Zhou_2016, Alvarez_2016, Lebedev_2016}. This involves grouping the weights for every neuron and attempting to sparsify each group by penalizing its $L_2$ norm. These techniques require all weights to be driven to zero before a node can be guaranteed to carry no information. In practice nodes are removed based on a threshold since this guarantee is difficult to meet. Because of this, regularization methods can be difficult to use as they require very precise tuning of the regularization and threshold terms.

The most similar technique to ours, Liu et al. \cite{Liu_2017},  uses $L_1$ regularization to drive the scale parameter in batch norm, $\gamma$, towards zero. This is similar in principle to our own experiments with batch norm. However, Liu et al. requires retraining after pruning in order to reconverge. We provide a more absolute condition to guarantee a node is off, eliminating the need for a retraining procedure and making node removal a more fluid process.

Our technique falls within the regularization category. Key differences in our approach involve special regularization of the bias for each neuron and a condition for node removal guaranteeing no effect on network output. Our condition is also more relaxed, utilizing the ``dead'' region in a node's activation function, instead of requiring the node's weights to be zero.

\subsection{Other approaches}
Several other approaches have appeared which do not fit into the categories of the previous two subsections. Many of these approaches focus on reducing precision as opposed to reducing the number of parameters. \cite{Hubara_2016, Vanhoucke_2011, Gupta_2015, Rastegari_2016}. As such, these approaches are largely orthogonal to our own work, and can be used in conjunction with our work in order to compound the reduction on memory and computation. One example of this approach is quantized and binarized neural networks \cite{Hubara_2016}, which take this approach to new levels by using $\{-1,1\}$ weights and an XOR to replace multiplication.

An additional noteworthy paper is that of Molchanov et al. \cite{Molchanov_2017_vd}. They achieve impressive results by sparsifing a network's connections during training using variational dropout. Again, in theory this work should be usable in conjunction with our own. 

\section{Methods}
\label{Methods}
\subsection{NodeDrop Condition}
In this section we describe the condition for identifying useless nodes in a network. Nodes in a neural network carry information by outputting values from some distribution. A node can only be useful if that node sometimes outputs a non-zero value. A node which is guaranteed to always output a constant value is a node which can only be used as an extra bias node for future layers. Moreover, if a node is guaranteed to always output the constant zero, this node is entirely useless and can be removed from the network without impact. This occurs in activation functions with a flat zero region. The popular rectified linear unit (ReLU) activation function contains such a flat zero region. This flat zero region causes the observed ``Dying ReLU'' effect, in which nodes become stuck in this flat region with zero gradients. We can therefore design a condition to identify when a node is useless by taking advantage of this effect. 


We propose the NodeDrop condition. 
\begin{enumerate}
\item Given a node with input vector $\vec{x} \in [0,1]^n$, a weight vector $\vec{w} \in \mathbb{R}^n$, bias b $\in \mathbb{R}$, and an activation function $\sigma$ such that $\sigma(v)=0 \,\, \forall v \leq 0$. 
\item We wish to find the condition under which this node is dead, $\sigma(\vec{w} \cdot \vec{x} + b )=0$ for all inputs $\vec{x}$. 
\item Since $\sigma(v)=0 \,\, \forall v \leq 0$, we simply need to find the condition under which $\vec{w} \cdot \vec{x} + b \leq 0$. We have constrained the inputs to be within $[0,1]$, $\vec{x} \in [0,1]^n$, so we have:
\[    \vec{w} \cdot \vec{x}+b  \leq \|\max(\vec{w},\vec{0})\|_1+b \leq \|\vec{w}\|_1 +b \]
\item Then, $\|\max(\vec{w}_i,0)\|_1+b \leq 0 \Rightarrow  \sigma(\vec{w} \cdot \vec{x} + b )=0$
\end{enumerate}
This leaves us with the NodeDrop condition:
\begin{equation} \label{eq:NodeDrop}
    \|\max(\vec{w}_i,0)\|_1+b \leq 0
\end{equation}


This condition can be applied to a fully connected layer, or in broader contexts such as filter weights of a convolutional layer. Because nodes which satisfy this condition are guaranteed to always output zero, they can be dropped from the network without affecting its output. Note that the weaker condition $\|\vec{w}\|_1 +b \leq 0$ can also be used, but identifies fewer nodes that can be dropped. 

The constraint on the activation function $\sigma(v)=0 \,\, \forall v \leq 0$ can be achieved using the standard ReLU activation of $\max(0, x)$. However ReLU does not guarantee that the output will fall between $0$ and $1$, a necessary condition if we want to apply NodeDrop to the next layer in the network. In the following section we will discuss an activation function for which the NodeDrop condition can be applied to both a layer and its following layer.

\subsection{Activation Function}
Supposing we want to apply NodeDrop to many or all layers of the network, we must use an activation function which possesses the appropriate flat zero region $(-\infty,0]$, and whose outputs are always between $0$ and $1$. The flat zero region guarantees the NodeDrop condition can be applied to the layer preceding activation, and the $[0,1]$ constraint on the output allows for the NodeDrop condition to be applied to the layer immediately following activation. These necessary constraints are reiterated below.

\begin{tabularx}{\columnwidth}{XX}
\vspace{-30pt}
\begin{equation} \label{eq:ActivFlat}
    \sigma(v)=0 \,\, \forall v \leq 0
\end{equation}
\vspace{-30pt}
&
\vspace{-30pt}
\begin{equation} \label{eq:Activ01}
    \sigma(v) \in [0,1] \,\, \forall v
\end{equation}
    \vspace{-30pt}
\end{tabularx}

If the outputs of a layer are guaranteed between 0 and 1 after activation, the inputs of the next layer will satisfy the conditions assumed in proving the NodeDrop condition. Many activation functions can satisfy these conditions, but none of the most popular activation functions satisfy both together. For example, the popular ReLU function satisfies the first condition in equation \ref{eq:ActivFlat} but not \ref{eq:Activ01}. Conversely, The popular sigmoid activation function satisfies equation \ref{eq:Activ01} but not \ref{eq:ActivFlat}.

One option is a clamped ReLU activation function, $\min(1, \max(0, x))$. This has two flat regions, $\sigma(v \geq 1) = 1, \sigma(v \leq 0) = 0$, and an intermediate region where $\sigma(v \in [0,1]) = v$. This does satisfy both of the NodeDrop conditions; however, we found that having two regions with zero gradients can lead to too many nodes being ``stuck'' at either $0$ or $1$ even at network initialization. Thus, we propose the SoftClampedReLU activation function, which is a combination between ReLU and inverted SoftPlus activations:
\begin{equation}
    \sigma(v)= max(0, 1- \frac{1}{\beta} log(1+e^{\beta (1-x)})
\end{equation} 

Intuitively this activation is much like a ClampedReLU, but has a soft gradient in the upper region. This upper region is not perfectly flat and so values do not become stuck at $\sigma(v)=1$. The $\sigma(v \leq 0)=0$ lower region is still perfectly flat, satisfying our flat region condition, equation (\ref{eq:ActivFlat}). This activation function is shown in figure \ref{fig:SoftClampedRelu}. In our experiments we use $\beta=10.0$.

\begin{figure}[H]
\vskip 0.0in
\begin{center}
\centerline{\includegraphics[width=0.6\columnwidth]{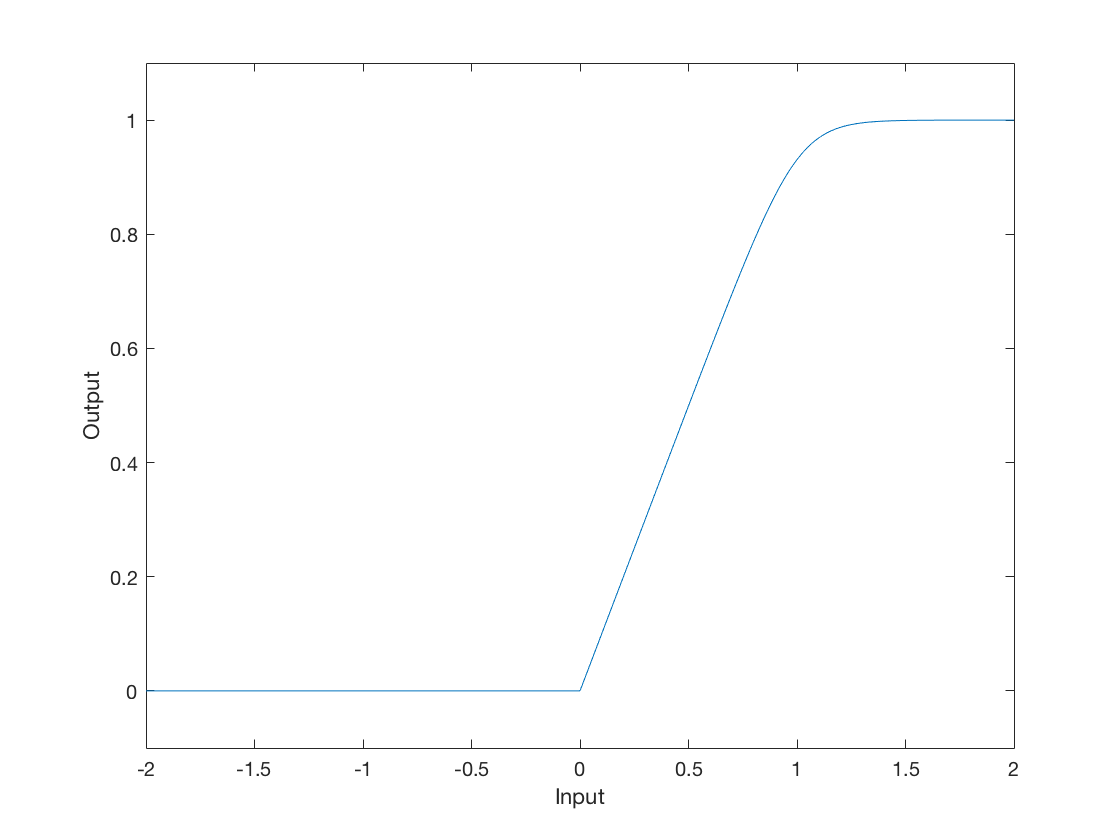}}
\caption{SoftClampedRelu activation function. Shown with $\beta=10$.}
\label{fig:SoftClampedRelu}
\end{center}
\vskip -0.2in
\end{figure}

\subsection{Regularization} \label{sec:reg}
The NodeDrop condition for identifying and eliminating useless nodes is powerful, but without encouragement, most trained networks will possess very few nodes satisfying the NodeDrop condition. Therefore, we add regularization during training to encourage features to satisfy the NodeDrop condition. 

To encourage $\|\max(\vec{w},\vec{0})\|_1+b = 0$ we can directly penalize its distance from zero:
\[ \lambda |\|\max(\vec{w},\vec{0})\|_1+b|\]
However, this is too close to the boundary of our dead region. 
Alternatively to encourage $\max(W,\vec{0})+b \leq 0$, we could penalize it directly:
\[\lambda ( \|\max(\vec{w},\vec{0})\|_1+b)\]
However, this causes the bias to tend toward negative infinity. 

Instead we penalize the distance from a negative constant, $-C$, given:
\[\lambda |\|\max(\vec{w},\vec{0})\|_1+b + C| = \lambda | \sum_i \max(\vec{w}_i, 0) + b + C|\]
This encourages $\|\max(\vec{w},\vec{0})\|_1+b = -C$, safely within the ``dead'' region, and without tending to negative infinity. As such, the choice of C is largely arbitrary; in our experiments we found a value of $1.0$ worked well, though other values worked just as well.

We can also write our regularization term as a small modification to standard $L_1$ regularization. For the case where $\|\max(\vec{w},\vec{0})\|_1+b \geq 0$, when a node is on, $\lambda | \|max(\vec{w}_i, 0)\|_1 + b + C| = \lambda (\|\max(\vec{w},\vec{0})\|_1 + \| b +C \|_1)$. We use this modified $L_1$ regularization given as:
\begin{equation}\label{eq:NodeDropRegularization}
    \lambda (\|\max(\vec{w},\vec{0})\|_1 + \| b +C \|_1)
\end{equation}
This is normal $L_1$ regularization with two adjustments. We use the $L_1$ norm of $\max(\vec{w},\vec{0})$ instead of $\vec{w}$ since this is a tighter bound given that $x \geq 0$. Instead of penalizing the bias as $\|b\|_1$, we penalize $\|b+C\|_1$. This modified $L_1$ regularization encourages biases to take bias values near $b=-C$, and weight values near 0. 

We use $L_1$ regularization because $L_2$ regularization does not work well in our context. For $L_2$ regularization on both the weights and the bias, it is cheaper to use multiple weights as a bias rather than the bias itself. That is, when $\sum_i^n w_i = b$, then $\|w\|_2 < \|b\|_2$.  This becomes worse when we penalize the distance of the bias from $-C$ rather than from zero, making the normal case of an active node with bias around $C$ quite costly. This encourages the network to use many nodes in the previous layer as an alternative to a bias, preventing us from removing those nodes even if they carry no information beyond that of a bias. 


\subsection{Extension to Batch Normalization} \label{sec:bn}
Many state of the art networks utilize batch normalization or one of its alternatives \cite{Ioffe_2015, ba_2016, Salimans_2016}. Here we consider our NodeDrop condition in a network with batch normalization. Batch normalization is given as follows:
\begin{align*}
\mu &= \frac{1}{m} \sum_{i=1}^m x_i \\
\sigma^2 &= \frac{1}{m} \sum_{i=1}^m (x_i - \mu)^2 \\
\hat{x}_i &= \frac{x_i - \mu}{\sqrt{\sigma^2 + \epsilon}} \\
y_i &= \gamma \hat{x}_i + \beta
\end{align*}
where the sum is over the batch of size $m$, and both $\gamma$ and $\beta$ are learned parameters.  Batch normalization is usually applied between the output of a layer and an activation function. 

To achieve a similar NodeDrop condition for batch normalization as in equation \ref{eq:NodeDrop}, we would like to determine when $y_i \leq 0$. We similarly require an activation with a flat zero region, but no longer require an input $\vec{x}$ between 0 and 1. Therefore, for our batch normalization NodeDrop (NodeDrop-BN) technique we are able to use the popular ReLU activation function. Our NodeDrop-BN condition is given in the following lemma. 

\begin{lemma}
\label{lemma:0}
$| \gamma | \sqrt{m} + \beta \leq 0 \implies y_i \leq 0 $.
\end{lemma}
\begin{proof}
\begin{align*}
\quad && \text{given }& \hat{x}_i && = \frac{x_i - \mu}{\sqrt{\sigma^2 + \epsilon}} &&\quad \quad\\
\quad && \implies& \hat{x}_i^2 && = \frac{(x_i - \mu)^2}{\sigma^2 + \epsilon} &&\quad \quad\\
\quad && \implies& \sum_{i=0}^m \hat{x}_i^2 &&= \frac{\sum_{i=0}^m(x_i - \mu)^2}{\sigma^2 + \epsilon} &&\quad\quad\\
\quad && & &&= \frac{m \sigma^2}{\sigma^2 + \epsilon} &&\quad\quad\\
\quad && \implies & \sum_{i=0}^m \hat{x}_i^2 && \leq m &&\quad\quad\\
\quad && \implies & -\sqrt{m} &&\leq \hat{x}_i \leq \sqrt{m} &&\quad\quad
\end{align*}
Together $|\hat{x}_i| \leq \sqrt{m}$ and $y_i = \gamma \hat{x}_i + \beta$ imply $y_i \leq | \gamma | \sqrt{m} + \beta$. Therefore 
\[ | \gamma | \sqrt{m} + \beta \leq 0 \implies y_i \leq 0 \]  
\phantom\qedhere
\end{proof}
\vspace{-20pt}

This gives us the NodeDrop-BN condition:
\begin{equation} \label{eq:NodeDrop_BatchNorm}
| \gamma | \sqrt{m} + \beta \leq 0
\end{equation}

Traditionally, batch normalization stores a running mean, $\mu$, and variance, $\sigma^2$ during training. These stored values are then used during testing. Our condition guarantees a node is always off during training, but does not guarantee a node will always be off during testing. We make the assumption that a node which is always off during training should also be off during testing. Thus, we can safely remove these nodes without impact. We experimentally validate this assumption in section \ref{sec:Experiments}. 

The condition in equation \ref{eq:NodeDrop_BatchNorm} implies that so long as we use an activation function where $\sigma(v)=0 \,\, \forall v \leq 0$ (for example ReLU), we can determine if a node is off using only the  batch normalization parameters, $\gamma$ and $\beta$, and the training batch size, $m$. Following the same methodology for regularization as in equation \ref{eq:NodeDropRegularization}, we define our NodeDrop-BN  $L_1$ regularization term as:
\begin{equation} \label{eq:NodeDropRegularizationBN}
    |\gamma|\sqrt{m}| + |\beta + C|
\end{equation}

$L_2$ regularization is generally applied to the layers before batch normalization. Unlike the vanilla NodeDrop regularization, $L_2$ regularization does not interfere with NodeDrop-BN technique, because the $L_2$ regularization applied to layers before a batch normalization has no effect on the output of the batch normalization layer.


\section{Experiments}
\label{sec:Experiments}

\begin{table*}[t]
\caption{MNIST Network Architectures: Number of Features by Layer}
\label{tab:MNISTarchitectures}
\vskip 0.0in
\begin{center}
\begin{small}
\begin{sc}
\begin{tabular}{lccccccc}
\toprule
Network Name & Layer 1  & Layer 2  & Layer 3  & Layer 4& Layer 5 & Layer 6 \\
\,\, & Conv2d  $3 \times 3$& Conv2d  $3 \times 3$ & Conv2d  $3 \times 3$  & Conv2d  $3 \times 3$& Dense & Output\\
\,\, & \,\,  & Maxpool $2\times 2$ & \,\,   & Maxpool $2\times 2$  &\,\,&\,\, \\
\midrule
Dense160 & 16 & 16 & 32 &32 &64 &10 \\
Dense240 & 24 & 24 & 48 &48 &96 &10 \\
Dense320 & 32 & 32 & 64 &64 &128 &10 \\
Dense480 & 48 & 48 & 96 &96 &192 &10 \\
Dense640 & 64 & 64 & 128 &128 &256 &10 \\
\bottomrule

\end{tabular}
\end{sc}
\end{small}
\end{center}
\vskip -0.1in
\end{table*}

\begin{figure*}[h]
\label{fig:MNISTaccsize}
\vskip 0.0in
\begin{center}
\centerline{\includegraphics[width=\textwidth]{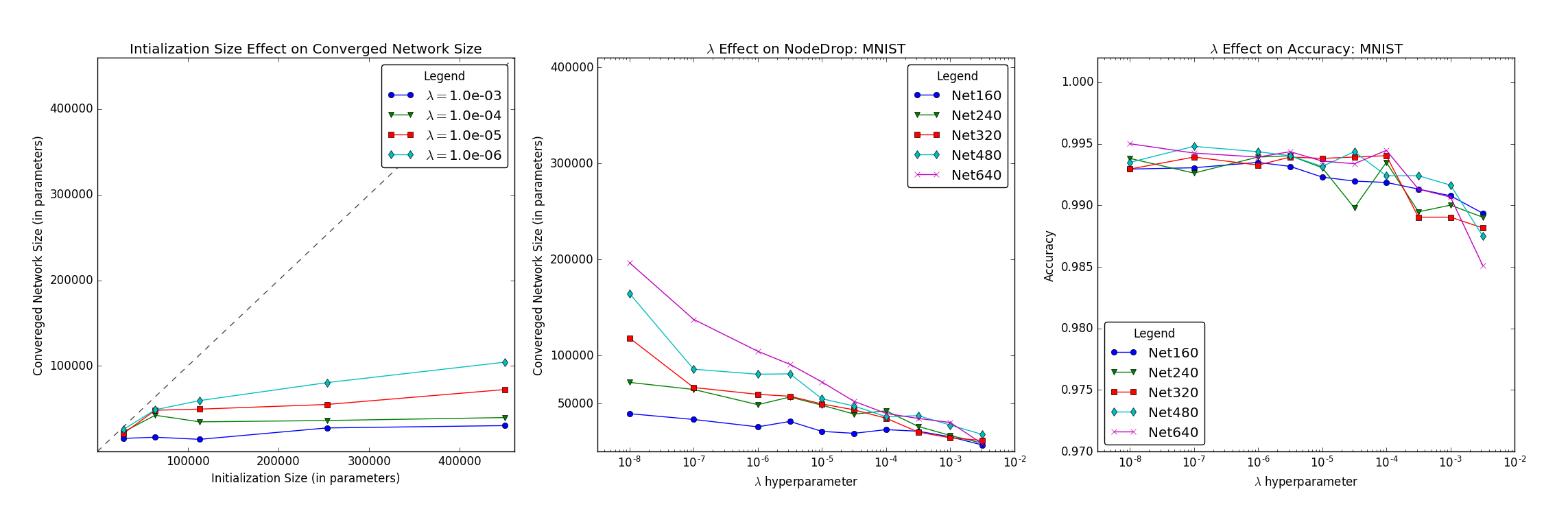}}
\caption{In the right and center figures, the $\lambda$ parameter values plotted on the y-axis are on a logarithmic scale.  We note that the performance and parameter reduction both maintain desirable levels for a large range of $\lambda$ values (over several orders of magnitude). This indicates the ease of tuning the NodeDrop technique.  In the leftmost figure, networks of different starting size converge to nearly the same size for a given $\lambda$. The dashed diagonal line represents networks without pruning. Note that increased initialization size has a slight effect on final size, as indicated by the slight upward slopes. This effect is greater for larger $\lambda$.}
\label{fig:sizesMNIST}
\end{center}
\vskip -0.2in
\end{figure*}

\begin{figure*}[h]
\vskip 0.0in
\begin{center}
\centerline{\includegraphics[width=0.8\textwidth]{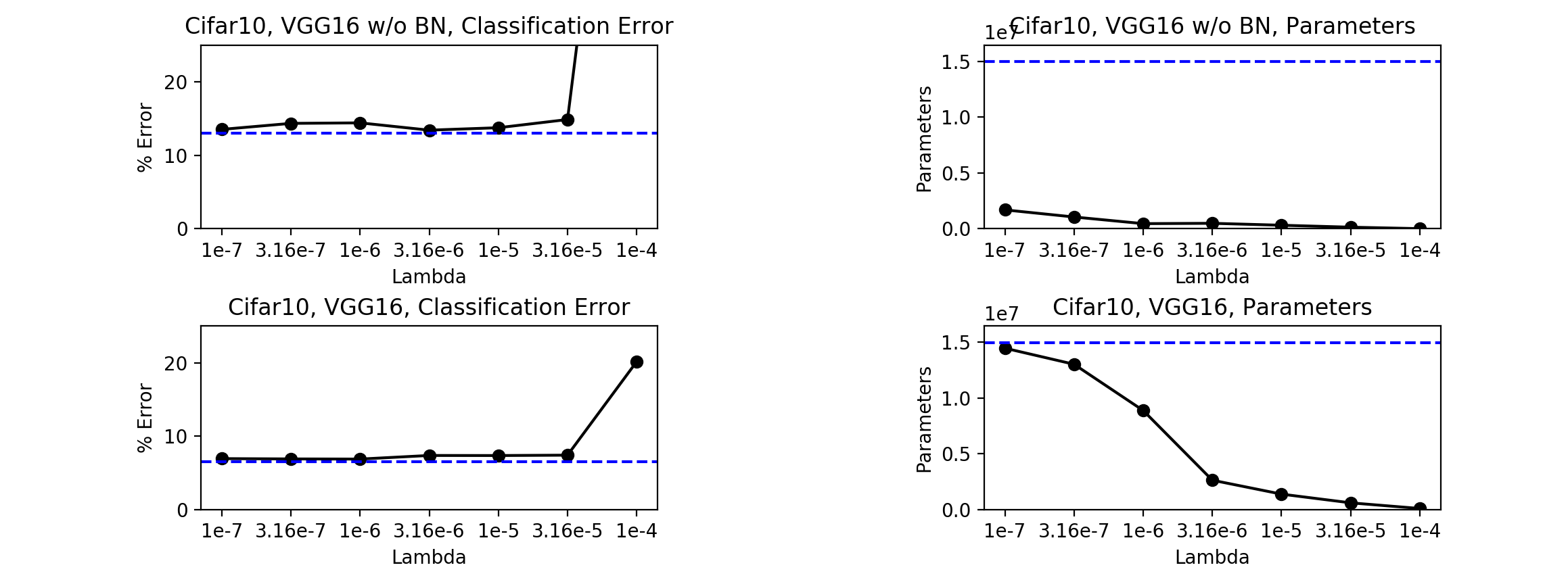}}
\caption{Results on CIFAR10 for VGG with and without Batch Normalization over a spread of $\lambda$ choices. Top Left: Classification error for VGG without Batch Normalization. Top Right: Final parameters after training using NodeDrop. Bottom Left: Classification error for VGG with Batch Normalization. Bottom Right: Final parameters after training using NodeDrop-BN. For both NodeDrop and NodeDrop-BN, a range of $\lambda$ values are acceptable. Baseline accuracy and network size is indicated by the dashed lines.}
\label{fig:cifarParams}
\end{center}
\vskip -0.2in
\end{figure*}

\begin{figure}[h]
\vskip -0.1in
\begin{center}
\centerline{\includegraphics[width=0.8\columnwidth]{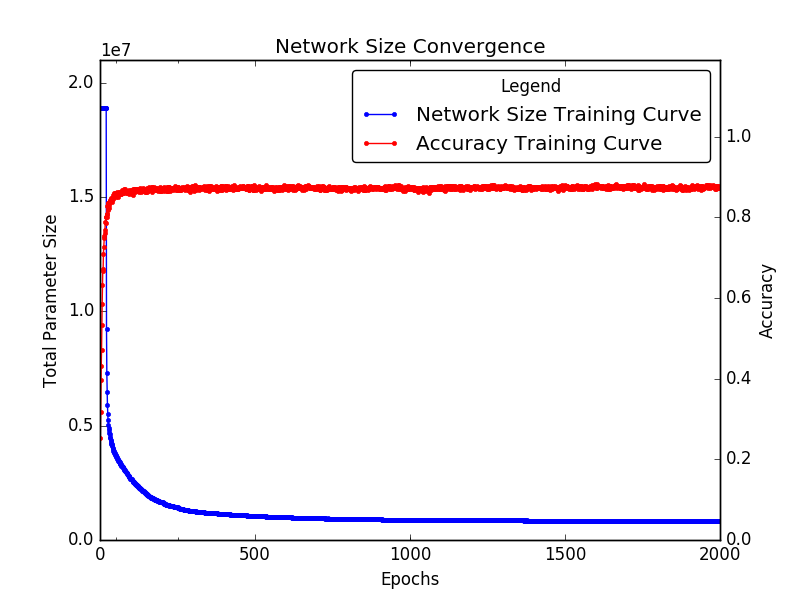}}
\label{fig:epochsCIFAR2}
\caption{Accuracy stabilizes after less than 100 epochs in this CIFAR10 run, indicating the NodeDrop technique does not delay performance convergence. Training for another 400 epochs helps maximize parameter reduction.}
\label{fig:sizesMNIST}
\end{center}
\vskip -0.4in
\end{figure}

\begin{table*}[t]
\caption{Cifar10 Classification Results}
\label{tab:cifar10}
\vskip 0.0in
\begin{center}
\begin{small}
\begin{sc}
\begin{tabular}{lccccccc}
\toprule
Network & $\lambda$ & Test Error & Parameters & Pruned \% & Factor & Nodes & Pruned \% \\
\midrule
\multirow{ 5}{*}{VGG 16 w/o BN} & Baseline & 13.01 & 15.04M & 0.0 & 1.0 & 4736 & 0.0 \\
& $1.0 \times 10^{-6}$ & 14.14 & 0.45M & 97.00 & 33.28 & 1115 & 76.46 \\
& $1.0 \times 10^{-5}$ & 13.27 & 0.31M & 97.96 & 48.98 & 859 & 81.9 \\
& \textbf{$3.2 \times 10^{-5}$} & \textbf{13.76} & \textbf{0.13M} & \textbf{99.12} & \textbf{114.00} & \textbf{612} & \textbf{87.08} \\
& $1.0 \times 10^{-4}$ & 90.00 & 0.0M & 100.0 & - & 0 & 100.0 \\ 
\midrule
\multirow{ 5}{*}{VGG 16} & Baseline & 6.50 & 15.04M & 0.0 & 1.0 & 4736 & 0.0 \\
& $1.0 \times 10^{-6}$ & 6.88 & 8.88M & 40.7 & 1.69 & 3624 & 23.48 \\
& $1.0 \times 10^{-5}$ & 7.36 & 1.39M & 90.75 & 10.81 & 1164 & 75.42 \\
& \textbf{$3.2 \times 10^{-5}$} & \textbf{7.41} & \textbf{0.61M} & \textbf{95.96} & \textbf{24.76} & \textbf{751} & \textbf{84.14} \\
& $1.0 \times 10^{-4}$ & 20.16 & 0.10M & 99.35 & 152.84 & 308 & 93.50 \\
\midrule
\multirow{ 4}{*}{DenseNet40 w/o BN} & Baseline      & 14.94 & 1.04M & 0.0 & 1.0 & 456 & 0.0 \\
& $1.0 \times 10^{-6}$                                                    & 15.21 & 0.66M & 35.69 & 1.55 & 363 & 20.39 \\
& $1.0 \times 10^{-5}$                                                    & 14.74 & 0.41M & 60.47 & 2.54 & 291 & 36.18 \\
& \textbf{$1.0 \times 10^{-4}$}                                           & \textbf{14.99} & \textbf{0.08M} & \textbf{91.96} & \textbf{12.43} & \textbf{154} & \textbf{66.22} \\
\midrule
\multirow{ 4}{*}{DenseNet40} & Baseline     & 6.80 & 1.05M & 0.0 & 1.0 & 456 & 0.0 \\
& $1.0 \times 10^{-6}$                                            & 7.13 & 0.99M & 4.19 & 1.04 & 447 & 1.97 \\
& $1.0 \times 10^{-5}$                                            & 6.75 & 0.98M & 5.67 & 1.06 & 443 & 2.85 \\
& \textbf{$1.0 \times 10^{-4}$}                                   & \textbf{7.79} & \textbf{0.55M} & \textbf{47.12} & \textbf{1.89} & \textbf{333} & \textbf{26.73} \\
\bottomrule
\end{tabular}
\end{sc}
\end{small}
\end{center}
\vskip -0.1in
\end{table*}

\begin{table*}[t]
\caption{Cifar100 Classification Results}
\label{tab:cifar100}
\vskip 0.0in
\begin{center}
\begin{small}
\begin{sc}
\begin{tabular}{lccccccc}
\toprule
Network & $\lambda$ & Test Error & Parameters & Pruned \% & Factor & Nodes & Pruned \% \\
\midrule
\multirow{ 2}{*}{VGG 16} & Baseline & 27.65 & 15.04M & 0.0 & 1.0 & 4736 & 0.0 \\
& $1.0\times 10^{-6}$ & 27.69 & 9.78M & 34.99 & 1.54 & 3914 & 17.35 \\
& \textbf{$1.0 \times 10^{-5}$} & \textbf{28.04} & \textbf{1.83M} & \textbf{87.82} & \textbf{8.21} & \textbf{1623} & \textbf{65.73} \\
& $1.0 \times 10^{-4}$ & 38.49 & 0.46M & 96.93 & 32.58 & 729 & 84.6 \\
\midrule
\multirow{ 4}{*}{DenseNet40} & Baseline & 26.5 & 1.05M & 0.0 & 1.0 & 456 & 0.0 \\
                             & $1.0 \times 10^{-6}$     & 26.92 & 1.05M & 2.27 & 1.02 & 451 & 1.09 \\
                             & \textbf{$1.0 \times 10^{-5}$}     & \textbf{27.01} & \textbf{1.03M} & \textbf{4.74} & \textbf{1.05} & \textbf{445} & \textbf{2.41} \\
                             & $1.0 \times 10^{-4}$     & 29.38 & 0.744M & 31.12 & 1.45 & 376 & 17.54 \\
\bottomrule
\end{tabular}
\end{sc}
\end{small}
\end{center}
\vskip -0.1in
\end{table*}

Having established a theoretical basis for the NodeDrop condition and regularization technique, we will now establish NodeDrop's practical viability as a method for shrinking networks. The NodeDrop technique requires two hyperparameters: $C$ and $\lambda$. The C value is unimportant, and can be set to almost any positive value without impacting results or parameter reduction. However, the $\lambda$ parameter is crucial in determining the balance between learning the objective and dropping nodes. Therefore, we closely examine the effect that choosing different $\lambda$ values has on both network performance and parameter reduction. We test many $\lambda$ values on the MNIST and CIFAR10 datasets. We also test a few $\lambda$ values on the CIFAR100 dataset.

The network initalization size should affect the number of nodes dropped. We show that if a network starts near optimal size, NodeDrop will maintain accuracy and only drop what few nodes it can. Furthermore, we show that if a network is grossly oversized at initialization, NodeDrop will drop many nodes and converge towards the same size as a smaller network initialization. This result is desirable, as it demonstrates NodeDrop is largely unaffected by poor layer size choices. NodeDrop uses $\lambda$ to determine the balance between performance and number of nodes utilized. Therefore, a network architect using NodeDrop can afford to initialize a large network, and remain confident that NodeDrop will eliminate needless nodes. Using the MNIST dataset, we demonstrate this ability by showing that networks will converge to the same size from multiple initialization sizes, for a fixed $\lambda$.

Many pruning methods require an increase in training time to be effective. The NodeDrop technique does not delay performance or accuracy convergence, but in order to allow the number of network nodes to converge, one must train for a longer time. We examine the training time required for this convergence with experiments on the CIFAR10 datasets.

Most importantly we test to ensure NodeDrop maintains performance and effectively drops nodes. We find that NodeDrop regularization does not affect a network's performance for a large swath of $\lambda$ values, only reducing testing accuracy if extreme $\lambda$ values are chosen. 

Furthermore, we demonstrate that NodeDrop is able to drop more than $100$x parameters from popular networks such as VGG16, while continuing to maintain classification accuracy on the CIFAR10 dataset. We test NodeDrop network performance and parameter reduction on MNIST, CIFAR10, and CIFAR100.

\subsection{MNIST Experiments}
\label{sec:MNIST_experiments}

The MNIST dataset \cite{mnist} provides an opportunity to perform a large number of experiments because of the datasets rapid accuracy convergence. Thus, we used this dataset to sweep across $\lambda$ values for five differently sized, but otherwise similar, network architectures, as shown in table \ref{tab:MNISTarchitectures}. We demonstrate NodeDrop's ability to rapidly converge to similarly sized networks from different starting sizes.

For all MNIST experiments we used a simple network design: four convolution layers and a single fully connected layer. We used 3x3 filters in all convolution layers, and performed max-pooling after every second convolution layer. We varied the width of the layers in order to test the effects of changing network initialization size. We did not investigate the effects of changing network depth, but suspect that prudent selection of network depth remains important. The network architectures are described in table \ref{tab:MNISTarchitectures}. The following consistent hyperparameters were used across all MNIST runs: learning rate $=1.0 \times 10^{-3}$, batch size $=1024$, optimizer $=Adam$, loss function$=$\textit{cross entropy}, epochs$=480$.

\subsubsection{Choosing Lambda}
\label{sec:chooselambdamnist}
Choosing an appropriate value for the NodeDrop's $\lambda$ parameter remains an important task. In order to prove that the NodeDrop technique remains robust for many selections of $\lambda$, we tested five different network initialization sizes to observe differences in convergence across $\lambda$ values. The network architectures and hyperparameters are discussed in section \ref{sec:MNIST_experiments}. We tested ten different $\lambda$ values between $\lambda=1.0\times 10^{-8}$ and $\lambda=1.0 \times 10^{-3}$. 

Our results indicate that easy tuning is a benefit of the NodeDrop technique. We found that $\lambda$ selections across orders of magnitude yielded desirable results, as shown in figure \ref{fig:MNISTaccsize}. For $\lambda > 10^{-4}$ we noticed a drop in MNIST accuracy, and for $\lambda < 1.0 \times 10^{-7}$ we judged there to be a significant sacrifice in parameter reduction. Choosing appropriate $\lambda$ will always be dependent on both application and loss function. Because of these MNIST experiments, we expect that the NodeDrop technique is robust for a large range of $\lambda$ selections. For a network designer using the popular cross-entropy loss objective function, as we did, we would suggest $\lambda=1.0 \times 10^{-5}$.

\subsubsection{Network Sizes}


In the previous section (\ref{sec:chooselambdamnist}) we experimentally observed that tuning the $\lambda$ parameter of the NodeDrop technique should not cause a network designer grief. In this section, we will experimentally observe that choosing initialization layer sizes should also prove easy. We use the same experiments from the previous section (\ref{sec:chooselambdamnist}), but instead plot the effect of initializing with differently sized networks. This plot, shown in figure \ref{fig:MNISTaccsize}, demonstrated that the NodeDrop technique will converge to a similar ``equilibrium'' from many differently sized initialization networks.  The size of the final network is instead mostly dependent on $\lambda$. A network designer should err towards too large a network in order to ensure desirable performance.



\subsection{CIFAR10 and CIFAR100 Experiments}

\subsubsection{Dataset}
The CIFAR dataset \cite{cifar10} consists of $32$x$32$ colored natural images. Both CIFAR10 and CIFAR100 are designed for classification, containing $10$ and $100$ classes respectively. There are $50,000$ training images and $10,000$ testing images for both. We adopt a standard data augmentation scheme where the training images are shifted and mirrored horizontally \cite{He_2016, Liu_2017}. 

\subsubsection{Architectures and Training}
We implement our technique on two standard models, VGG \cite{Simonyan_2014} and DenseNet \cite{Huang_2017}. Our VGG network is a slight variant of the standard VGG16 model. We follow the standard modification of VGG for CIFAR \cite{Liu_2017, Molchanov_2017_vd}, by removing the $3$ final fully connected layers of size $4096$ and instead using only a single fully connected layer of size $512$. We train the network using SGD with momentum of $0.9$. The network is trained for $200$ epochs with an initial learning rate of 0.1 which is decayed by $0.1$ at epochs 80 and 130.  We tested both with and without batch normalization, and discovered that batch normalization is necessary for the large VGG16 initialization when applied to the more difficult CIFAR100 dataset. Therefore results without batch normalization are excluded for CIFAR100.

For DenseNet we implement the standard DenseNet-40 given in the original paper with $L=40$ and $k=12$. We train the model as per the original paper with SGD and momentum 0.9. The network is trained for $300$ epochs with an initial learning rate of 0.1 and is decayed by 0.1 at epochs 150 and 225. As with VGG we found that the CIFAR100 dataset required batch normalization, but we were again able to train a variant on CIFAR10 without batch normalization.  

\subsubsection{Lambda Parameter Tests}
As with the MNIST experiments, we tested a range of $\lambda$'s on CIFAR10 in order to determine the choices which suit the network and dataset well. Furthermore, here we test NodeDrop-BN, which was not tested in the MNIST experiments. Results for VGG on CIFAR10 with varying choices of $\lambda$ are shown in figure \ref{fig:cifarParams}. 

For the case without batch normalization our network maintains performance and prunes a large number of nodes over many choices of $\lambda$. As with the MNIST case, this suggests that choosing $\lambda$ is relatively easy. All choices of $\lambda \leq 3.2 \times 10^{-5}$ achieved high performance with significant pruning. For $\lambda \geq 1.0 \times 10^{-4}$ the regularization parameter proved too high, causing an entire layer to turn off, which in turn caused the network to turn off all other layers.

For NodeDrop-BN, we find that $\lambda \leq 3.2\times 10^{-5}$ is appropriate for maintaining performance. However, NodeDrop-BN requires more precise tuning than NodeDrop, as only $\lambda \geq 3.2 \times 10^{-6}$ achieved desirable parameter reduction. Based on the above results we continue to recommend an initial lambda setting of $\lambda = 1 \times 10^{-5}$ for the cross-entropy loss objective function.

\subsubsection{Network Convergence Time}
Sometimes it is important to avoid needlessly extending training time. In this section we analyze NodeDrop's effect on training time. Using $\lambda=10^{-5}$, we train a network for 2000 epochs in order to observe network parameter and performance convergence over time. This experiment used the VGG16 network without batch normalization on the CIFAR10 dataset. Our results, shown in figure \ref{fig:epochsCIFAR2}, indicate that while accuracy convergence is not delayed by the NodeDrop technique, one will need to wait longer to maximize NodeDrop's parameter reduction.

\subsubsection{Parameter Reduction}

Results for CIFAR10 and CIFAR100 are given in tables \ref{tab:cifar10} and \ref{tab:cifar100} respectively. We highlight the rows which provide the highest parameter reduction while maintaining high accuracy. 

For the VGG network we are able to drop a significant number of parameters without degradation to the accuracy of the network. For NodeDrop-BN, we can prune $95$ percent of the parameters for CIFAR10 and $88$ percent for CIFAR100. For vanilla NodeDrop, we can prune $99$ percent of the parameters on CIFAR10. This suggests that VGG is a significantly oversized network for application to the CIFAR datasets. 

It is more difficult to prune nodes from the DenseNet architecture than for VGG. We are only able to prune approximately $5$ percent of the parameters from DenseNet on CIFAR100. We believe this suggests that the DenseNet architecture is already well sized for CIFAR100. DenseNet starts at around $1$ million parameters, which is close to the number of remaining parameters after our best case pruning of the VGG network.

\section{Conclusion and Outlook}
\label{Conclusion}
In this paper, we proposed the novel NodeDrop technique for reducing parameters in neural networks. The NodeDrop technique consists of a condition for identifying nodes which are guaranteed to carry no information, and a regularization term to encourage this condition to be met. We also propose a modified version of NodeDrop, NodeDrop-BN, for use in networks with batch normalization. Experiments on the MNIST and CIFAR10 datasets show that NodeDrop does not significantly increase training time, and facilitates network design with the easily tuneable hyperparameter $\lambda$. With experiments on MNIST, CIFAR10, and CIFAR100 datasets, using VGG16 and DenseNet architectures, we demonstrate that NodeDrop compares favorably with other parameter reduction techniques. NodeDrop reduces the number of parameters in a network by up to a factor of 114x. We hope that NodeDrop and NodeDrop-BN will prove useful in neural network design, and will help to make the implementation of neural networks on constrained systems more practical.


\clearpage

\bibliography{main}
\bibliographystyle{icml2019}

\end{document}